\documentclass[twoside]{article}

\usepackage[accepted]{aistats2020}
\usepackage{graphicx}
\usepackage[toc,page]{appendix}
\usepackage{xcolor}
\usepackage{bm}
\usepackage{amsmath,amsthm}
\usepackage{amssymb}
\usepackage{tikz}
\usepackage{dsfont}
\usepackage{booktabs}
\usepackage{multirow}
\usepackage{float}
\usepackage{array}
\usepackage{makecell}
\usepackage{xparse}

\usepackage[round]{natbib}

\theoremstyle{remark}
\newtheorem*{remark}{Remark}
\theoremstyle{definition}
\newtheorem{definition}{Definition}
\newtheorem{proposition}{Proposition}

\let\emptyset\varnothing

\tikzset{
  treenode/.style = {shape=rectangle, rounded corners,
                     draw, align=center,
                     top color=white,
                     bottom color=blue!20},
  root/.style     = {treenode, font=\Large,
                     bottom color=red!30},
  env/.style      = {treenode, font=\ttfamily\normalsize},
  dummy/.style    = {circle,draw},
  level 1/.style={sibling distance = 16em},
  level 2/.style={sibling distance = 8em},
}
\usepackage[linesnumbered,ruled,vlined]{algorithm2e}

\usepackage{cleveref}

\crefname{appsec}{Appendix}{Appendices}

\newcommand\norm[1]{\left\lVert#1\right\rVert}
\newcommand{\todo}[1]{\textcolor{red}{\textbf{Fix:} \emph{#1}}}
\newcommand{\indep}{\perp \!\!\! \perp}
\newcommand*{\medcap}{\mathbin{\scalebox{1.3}{\ensuremath{\cap}}}}
\newcommand*{\medcup}{\mathbin{\scalebox{1.3}{\ensuremath{\cup}}}}
\DeclareMathOperator*{\argmax}{arg\,max}
\DeclareMathOperator{\diag}{diag}
\DeclareMathOperator{\Var}{Var}

\NewDocumentCommand{\change}{}{\textcolor{red}{\textbf{Changes: }}}
\begin{document}

%

%

\runningtitle{Additive Tree-Structured Covariance Function}
\twocolumn[

\aistatstitle{Additive Tree-Structured Covariance Function for Conditional Parameter Spaces in Bayesian Optimization}

\aistatsauthor{ Xingchen Ma \And Matthew B.\ Blaschko}

\aistatsaddress{ ESAT-PSI, KU Leuven, Belgium \And ESAT-PSI, KU Leuven, Belgium} 
]

\begin{abstract}

Bayesian optimization (BO) is a sample-efficient global optimization algorithm for black-box functions which are expensive to evaluate. Existing literature on model based optimization in conditional parameter spaces are usually built on trees. In this work, we generalize the additive assumption to tree-structured functions and propose an additive tree-structured covariance function, showing improved sample-efficiency, wider applicability and greater flexibility. Furthermore, by incorporating the structure information of parameter spaces and the additive assumption in the BO loop, we develop a parallel algorithm to optimize the acquisition function and this optimization can be performed in a low dimensional space. We demonstrate our method on an optimization benchmark function, as well as on a neural network model compression problem, and experimental results show our approach significantly outperforms the current state of the art for conditional parameter optimization including SMAC, TPE and Jenatton et al. (2017).

\end{abstract}

\section{INTRODUCTION}
\label{sec:introduction}

In many applications, we are faced with the problem of optimizing an expensive black-box function and we wish to find its optimum using as few evaluations as possible. \textit{Bayesian Optimization} (BO) \citep{jones1998} is a global optimization technique, which is specially suited for these problems. BO has gained increasing attention in recent years \citep{srinivas2010,brochu2010,hutter2011,shahriari2016,frazier2018} and has been successfully applied to sensor location \citep{srinivas2010}, hierarchical reinforcement learning \citep{brochu2010}, and automatic machine learning \citep{klein2017}.

In the general BO setting, we aim to solve the following problem:
\begin{equation*}
    \min_{\bm{x} \in \mathcal{X} \subset \mathbb{R}^d} f(\bm{x}),
\end{equation*}
where $\mathcal{X}$ is the parameter space and $f$ is a black-box function which is expensive to evaluate.
Typically, the parameter space $\mathcal{X}$ is treated as structureless, however, for many practical applications, there exists a conditional structure in $\mathcal{X}$: 
\begin{equation}
\label{eq:condition-dep}
    f(\bm{x} \mid \bm{x}_{\mathcal{I}_A}) = f(\bm{x}_{\mathcal{I}_B} \mid \bm{x}_{\mathcal{I}_A}),
\end{equation}
where the index sets $\mathcal{I}_A=\{ a_1,\dots,a_k \}$, $\mathcal{I}_B=\{ b_1,\dots,b_m \}$ and $\mathcal{I}_A \cup \mathcal{I}_B$ are subsets of $\mathcal{I}_D=\{ 1,\dots, d \}$.
Intuitively, \Cref{eq:condition-dep} means given the value of $\bm{x}_{\mathcal{I}_A}$, the value of $f(\bm{x})$ remains unchanged after removing $\bm{x}_{\mathcal{I}_D \setminus (\mathcal{I}_A \cup \mathcal{I}_B)}$. 
Here we use set based subscripts to denote the restriction of $\bm{x}$ to the corresponding indices. 

This paper investigates optimization problems where the parameter space exhibits such a conditional structure. In particular, we focus on a specific instantiation of the general conditional structure in \Cref{eq:condition-dep}: Tree-structured parameter spaces, which are also studied in \citet{jenatton2017}. Many problems fall into this category, for example, when fitting \textit{Gaussian Processes} (GPs), we need to choose from several covariance functions and subsequently set their continuous hyper-parameters. 
Different covariance functions may share some hyper-parameters, such as the signal variance and the noise variance \citep{rasmussen2006}.

By exploring the properties of this tree structure, we design an \textit{additive tree-structured} (Add-Tree) covariance function, which enables information sharing between different data points under the \textit{additive assumption}, and allows GP to model $f$ in a sample-efficient way. 
Furthermore, by including the tree structure and the additive assumption
in the BO loop, we develop a parallel algorithm to optimize the acquisition function, making the overall execution faster. Our proposed method also helps to alleviate the curse of dimensionality through two advantages: (i) we avoid modeling the response surface directly in a high-dimensional space, and (ii) the acquisition optimization is also operated in a lower-dimensional space.

In the next section, we will briefly review BO together with the literature related to optimization in a conditional parameter space. 
In Section~\ref{sec:problem_formulation}, we formalize the family of objective functions that can be solved using our approach. We then present our Add-Tree covariance function in Section~\ref{sec:add-tree-kernel}. In Section~\ref{sec:bo}, we give the inference procedure and BO algorithm using our covariance function. We then report a range of experiments in Section~\ref{sec:experiment}. Finally, we conclude in Section~\ref{sec:conclusion}.

\subsection{RELATED WORK}
\label{sec:related_work}

\subsubsection{Bayesian Optimization}
\label{subsec:bo}

BO has two major components. The first one is a probabilistic regression model used to fit the response surface of $f$. Popular choices include GPs \citep{brochu2010}, random forests \citep{hutter2011} and adaptive Parzen estimators \citep{bergstra2011}. We refer the reader to \citet{rasmussen2006} for the foundations of Gaussian Processes. The second one is an acquisition function $u_{t-1}$ which is constructed from this regression model and is used to propose the next evaluation point. Popular acquisition functions include the \textit{expected improvement} (EI) \citep{jones1998}, \textit{knowledge gradient} (KG) \citep{frazier2009}, \textit{entropy search} (ES) \citep{hennig2012} and \textit{Gaussian process upper confidence bound} (GP-UCB) \citep{srinivas2010}. 

One issue that often occurs in BO is, in high-dimensional parameter spaces, its performance may be no better than random search \citep{wang2013,li2016}. This deterioration is due to high uncertainty in fitting a regression model due to the curse of dimensionality \citep[Ch.~2]{gyorfi2002}, which in turn leads to pure-explorational behavior of BO. 
This will further cause inefficiency in the acquisition function, making the proposal of the next data point behave like random selection. 
Standard GP-based BO ignores the structure in a parameter space, and fits a regression model in $\mathbb{R}^d$. By leveraging this structure information, we can work in a low-dimensional space $\mathbb{R}^m$ (recall \Cref{eq:condition-dep}) instead of $\mathbb{R}^d$.



\subsubsection{Conditional Parameter Spaces}
\label{subsec:cps}

Sequential Model-based Algorithm Conﬁguration (SMAC)~\citep{hutter2011} and Tree-structured Parzen Estimator Approach (TPE)~\citep{bergstra2011} are two popular non-GP based optimization algorithms that are aware of the conditional structure in $\mathcal{X}$, however, they lack favorable properties of GPs: uncertainty estimation in SMAC is non-trivial and the dependencies between dimensions are ignored in TPE. Additionally, neither of these methods have a particular sharing mechanism, which is valuable in the low-data regime. 

In the category of GP-based BO, which is our focus in this paper, \citet{hutter2013} proposed a covariance function that can explicitly employ the tree structure and share information at those categorical nodes. However, their specification for the parameter space is too restrictive and they require the shared node to be a categorical variable.  By contrast, we allow shared variables to be continuous (see Section~\ref{sec:add-tree-kernel}). \citet{swersky2014a} applied the idea of \citet{hutter2013} in a BO setting, but their method still inherits the limitations of \citet{hutter2013}.  Another covariance function to handle tree-structured dependencies is presented in \citet{levesque2017}. In that case, they force the similarity of two samples from different condition branches to be zero and the resulting model can be transformed into several independent GPs.  We perform a comparison to an independent GP baseline in Section~\ref{sec:synthetic-function}. In contrast to Add-Tree, the above approaches either have very limited applications, or lack a sharing mechanism. 
\citet{jenatton2017} presented another GP-based BO approach, where they handle  tree-structured dependencies by introducing a weight vector linking all sub-GPs, and this introduces an explicit sharing mechanism. Although \citet{jenatton2017} overcame the above limitations, the enforced linear relationships between different paths make their semi-parametric approach less flexible compared with our method. We observe in our experiments that this can lead to a substantial difference in performance.

\section{PROBLEM FORMULATION}
\label{sec:problem_formulation}

We begin by summarizing notation used in this paper. Let $\mathcal{T}=(V,E)$ be a tree, in which $V$ is the set of vertices, $E$ is the set of edges, $P=\{p_i\}_{1\leq i \leq |P|}$ be the set of leaves and $r$ be the root of $\mathcal{T}$ respectively, $\{ l_i \}_{1\leq i \leq |P|}$ be the ordered set of vertices on the path from $r$ to the $i$-th leaf $p_{i}$, and $h_i$  be the number of vertices along $l_i$ (including $r$ and $p_i$). To distinguish an objective function defined on a tree-structured parameter space from a general objective function, we use $f_{\mathcal{T}}$ to indicate our objective function. In what follows, we will call $f_{\mathcal{T}}$ a~\textit{tree-structured function}.

To formalize the family of problems that can be solved with our method, we start with some definitions.

\begin{definition}[Tree-structured parameter space]
\label{def:TreeStructuredParameterSpace}
A \emph{tree-structured parameter space} $\mathcal{X}$ is associated with a tree $\mathcal{T}=(V,E)$. For any $v \in V$, $v$ is associated with a bounded continuous variable of $\mathcal{X}$; the set of outgoing edges $E_v$ of $v$ represent one categorical variable of $\mathcal{X}$ and each element of $E_v$ represents a specific setting of the corresponding categorical variable.
\end{definition}

\begin{definition}[Tree-structured function]
\label{def:TreeStructuredFunction}
A \emph{tree-structured function} $f_{\mathcal{T}}: \mathcal{X} \rightarrow \mathbb{R}$ is defined on a $d$-dimensional tree-structured parameter space $\mathcal{X}$. The $i$-th leaf $p_i$ is associated with a function $f_{p_i,\mathcal{T}}$ of the variables associated with the vertices along $l_i$. $f_{\mathcal{T}}$ is called \emph{tree-structured} if for every leaf of the tree-structured parameter space
\begin{equation}
\label{eq:def-tree-structured-function}
f_{\mathcal{T}}(\bm{x}) := f_{p_j,\mathcal{T}}(\bm{x}\vert_{l_j}) ,
\end{equation}
where $p_j$ is selected by the categorical values of $\bm{x}$ and $\bm{x}\vert_{l_j}$ is the restriction of $\bm{x}$ to $l_j$.
\end{definition}


To aid in the understanding of a tree-structured function (and subsequently our proposed Add-Tree covariance function), we depict a simple tree-structured function in Figure~\ref{fig:example}. 
The outgoing edges of $r$ represent the categorical variable $t\in \{1,2\}$ and the settings of $t$ are shown around these two edges. Vertices $r,p_1,p_2$ are associated with bounded variables $\bm{v}_r \in [-1,1]^2, \bm{v}_{p_1} \in [-1,1]^2, \bm{v}_{p_2} \in [-1,1]^3$ respectively and leaves $p_1,p_2$ are associated with two functions shown in Figure~\ref{fig:example}. In Definition~\ref{def:TreeStructuredFunction}, the restriction of one input to a path means we collect variables associated with the vertices along that path and concatenate them using a fixed ordering.  
For example, in Figure~\ref{fig:example}, let $\bm{x} \in \mathcal{X}$ be an 8-dimensional input, then the restriction of $\bm{x}$ to path $l_1$ is a 4-dimensional vector. The function illustrated in Figure~\ref{fig:example} can be compactly written down as:
\begin{align} \label{eq:FigureExampleTreeFunction}
    f_{\mathcal{T}}(\bm{x}) = \mathds{1}_{t=1} f_{p_1,\mathcal{T}}(\bm{x}\vert_{l_1}) + \mathds{1}_{t=2} f_{p_2,\mathcal{T}}(\bm{x}\vert_{l_2}) , 
\end{align}
where $\bm{x}$ is the concatenation of $(\bm{v}_r,\bm{v}_{p_1},\bm{v}_{p_2},t)$, $\mathds{1}$ denotes the indicator function, $f_{p_1,\mathcal{T}}(\bm{x}\vert_{l_1})=\norm{\bm{v}_r}^2 + \norm{\bm{v}_{p_1}}^2$ and $f_{p_2,\mathcal{T}}(\bm{x}\vert_{l_2}) = \norm{\bm{v}_r}^2 + \norm{\bm{v}_{p_2}}^2$.

\begin{figure}[t]
\centering
\begin{tikzpicture}[
    grow                    = right,
    edge from parent/.style = {draw, -latex},
    root/.style     = {treenode, font=\large, bottom color=red!30},
    every node/.style       = {font=\footnotesize},
    level distance          = 10em,
    level 1/.style={sibling distance = 8em},
    sloped
    ]
  \node [root] {$r, \bm{v}_r \in [-1,1]^2$}
    child { node [env] {$p_1, \bm{v}_{p_1} \in [-1,1]^2, f_{p_1, \mathcal{T}} = \norm{\bm{v}_r}^2 + \norm{\bm{v}_{p_1}}^2 $} edge from parent node[above] {t=1}  }
    child { node [env] {$p_2, \bm{v}_{p_2} \in [-1,1]^3, f_{p_2, \mathcal{T}} = \norm{\bm{v}_r}^2 + \norm{\bm{v}_{p_2}}^2 $} edge from parent node[below] {t=2}  };
\end{tikzpicture}
\caption{A Simple Tree-Structured Function} 
\label{fig:example}
\end{figure}
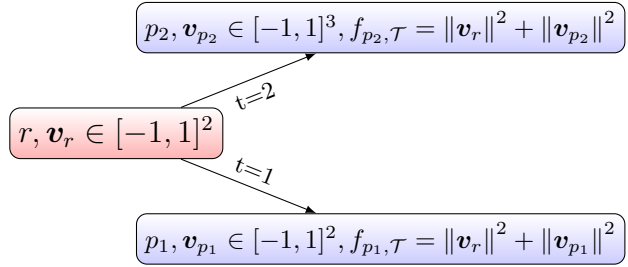


A tree-structured function $f_{\mathcal{T}}$ is actually composed of several smaller functions $\{ f_{p_i,\mathcal{T}} \}_{1 \leq i \leq |P|}$, given a specific setting of the categorical variables, $f_{\mathcal{T}}$ will return the associated function at the $i$-th leaf. To facilitate our description in the following text, we define the \textit{effective dimension} in Definition~\ref{def:EffectiveDimension}.

\begin{definition}[Effective dimension]
\label{def:EffectiveDimension}
The \emph{effective dimension} of a tree-structured function $f_{\mathcal{T}}$ at the $i$-th leaf $p_i$ is the sum of dimensions of the variables associated with the vertices along $l_i$. 
\end{definition}

\begin{remark}
The effective dimension of $f_{\mathcal{T}}$ can be much smaller than the dimension of $\mathcal{X}$. 
\end{remark}

For the tree-structured function depicted in Figure~\ref{fig:example}, the dimension of $\mathcal{X}$ is $d=8$ and the effective dimension at $p_1$ and $p_2$ are 4 and 5 respectively. Particularly, if $\mathcal{T}$ is a perfect binary tree, in which each vertex is associated with a 1-dimensional continuous variable, the effective dimension of $f_{\mathcal{T}}$ at every leaf is the depth $h$ of $\mathcal{T}$, while the dimension of $\mathcal{X}$ is $3 \cdot 2^{h-1}-2$.
If the tree structure information is thrown away, we have to work in a much higher dimensional parameter space. 
It is known that in high dimensions, BO behaves like random search, which violates the entire purpose of model based optimization. 

Now we have associated the parameter space $\mathcal{X}$ with a tree structure, which enables us to work in the low-dimensional effective space. How can we leverage this tree structure to optimize $f_{\mathcal{T}}$? Recalling $f_{\mathcal{T}}$ is a collection of $|P|$ functions $\{ f_{p_i,\mathcal{T}} \}_{1 \leq i \leq |P|}$, a trivial solution is using $|P|$ independent GPs to model each $f_{p_i, \mathcal{T}}$ separately. 
In BO settings, we are almost always in low data regime because black-box calls to $f_{\mathcal{T}}$ are expensive (e.g.\ the cost of training and evaluating a machine learning model). Modeling $f_{\mathcal{T}}$ using a collection of GPs is obviously not an optimal way because we discard the correlation between $f_{p_i,\mathcal{T}}$ and $f_{p_j,\mathcal{T}}$ when $i \neq j$. How to make the most of the observed data, especially how to share information between data points coming from different leaves remains a crucial question. In this paper, we assume additive structure within each $f_{p_i,\mathcal{T}}$ for $i=1,\cdots,|P|$. More formally, $f_{p_i,\mathcal{T}}$ can be decomposed in the following form:
\begin{equation}
\label{eq:additive-assumption}
    f_{p_i,\mathcal{T}}(\bm{x}) = \sum_{j=1}^{h_i} f_{ij}(v_{ij})
\end{equation}
where $v_{ij}$ is the associated variable on the $j$-th vertex along $l_i$. 
Additive assumption has been extensively studied in GP literature \citep{duvenaud2011,kandasamy2015,gardner2017,rolland2018} and is a popular way for dimension reduction\citep{gyorfi2002}.
We note the tree-structured function discussed in this paper is a generalization of the objective function presented in these publications and our additive assumption in \Cref{eq:additive-assumption} is also a generalization of the additive structure considered previously. For example, the additive function discussed in \citet{kandasamy2015} can be viewed as a tree-structured function the associated tree of which has a branching factor of 1, i.e.\ $|P|=1$.
Our generalized additive assumption will enable an efficient sharing mechanism as we develop in Section~\ref{sec:add-tree-kernel}.

\section{THE ADD-TREE COVARIANCE FUNCTION}
\label{sec:add-tree-kernel}

In this section, we describe how we use the tree structure and the additive assumption to design a covariance function, which is sample-efficient in low-data regime. We start with the definition of the Add-Tree covariance function (Definition~\ref{def:Add-Tree}), then we show the intuition (\Cref{eq:joint-distributation}) behind this definition and present an algorithm (Algorithm~\ref{algo:kernel-construction}) to automatically construct an Add-Tree covariance function from the specification of a tree-structured parameter space, finally a proof of the validity of this covariance function is given. 


\begin{definition}[Add-Tree covariance function]
\label{def:Add-Tree}
For a tree-structured function $f_{\mathcal{T}}$, let $\bm{x}_{i'}$ and $\bm{x}_{j'}$  be two inputs of $f_{\mathcal{T}}$, $p_i$ and $p_j$ be the corresponding leaves, $a_{ij}$ be the lowest common ancestor (LCA) of $p_i$ and $p_j$, $l_{ij}$ be the path from $r$ to $a_{ij}$. A covariance function $k_{f_{\mathcal{T}}}: \mathcal{X} \times \mathcal{X} \rightarrow \mathbb R $ is said to be an \emph{Add-Tree covariance function} if for each $\bm{x}_{i'}$ and $\bm{x}_{j'}$
\begin{align}
\label{eq:add-tree-kernel-1}
    k_{\mathcal{T}}(\bm{x}_{i'},\bm{x}_{j'}) := & k_{l_{ij}}(\bm{x}_{i'}\vert_{l_{ij}},\bm{x}_{j'}\vert_{l_{ij}})  \nonumber \\
    = & \sum_{v \in l_{ij}} k_v(\bm{x}_{i'}\vert_{v},\bm{x}_{j'}\vert_{v})
\end{align}
where $\bm{x}_{i'}\vert_{l_{ij}}$ is the restriction of $\bm{x}_{i'}$ to the variables along the path $l_{ij}$, and 
$k_v$ is any positive semi-definite covariance function on the continuous variables appearing at a vertex $v$ on the path $l_{ij}$. We note the notation $l_{ij}$ introduced here is different from the notation $l_i$ introduced in the beginning of \Cref{sec:problem_formulation}.
\end{definition}


To give the ideas behind the Add-Tree family of covariance functions, we take the tree-structured function illustrated in Figure~\ref{fig:example} (\Cref{eq:FigureExampleTreeFunction}) as an example.\footnote{To simplify the presentation, we use a two-level tree structure in this example. The covariance function, however, generalizes to tree-structured functions of arbitrary depth (Algorithm~\ref{algo:kernel-construction}).} 
Let $X_1 \in \mathbb{R}^{n_1 \times d_1} $ and $X_2 \in \mathbb{R}^{n_2 \times d_2}$ be the inputs from $l_1$ and $l_2$, where $d_1=2+2$ and $d_2=2+3$ are the effective dimensions of $f_{\mathcal{T}}$ at $p_1$ and $p_2$ respectively. Denote the latent variables associated to the decomposed functions\footnote{On functions and latent variables, one can refer to~\citet[chap.~2]{rasmussen2006}} at $r$, $p_1$ and $p_2$ by $\bm{f}_r \in \mathbb{R}^{n_1+n_2}$, $\bm{f}_1 \in \mathbb{R}^{n_1}$ and $\bm{f}_2 \in \mathbb{R}^{n_2}$, respectively. Reordering and partition $\bm{f}_r$ into two parts corresponding to $p_1$ and $p_2$, so that 
\begin{equation*}
    \bm{f}_{r} = 
    \begin{bmatrix}
           \bm{f}_{r}^{(1)} \\
           \bm{f}_{r}^{(2)}
    \end{bmatrix}, 
    \bm{f}_{r}^{(1)} \in \mathbb{R}^{n_1}, 
    \bm{f}_{r}^{(2)} \in \mathbb{R}^{n_2}.
\end{equation*}
Let the gram matrix corresponding to $\bm{f}_r, \bm{f}_1, \bm{f}_2$ be $K_r \in \mathbb{R}^{(n_1+n_2) \times (n_1+n_2)}, K_1 \in \mathbb{R}^{n_1 \times n_1}, K_2 \in \mathbb{R}^{n_2 \times n_2}$. W.l.o.g, let the means of $\bm{f}_r,\bm{f}_1,\bm{f}_2$ be $\bm{0}$. By the additive assumption in \Cref{eq:additive-assumption}, the latent variables corresponding to the associated functions at $p_1$ and $p_2$ are $\bm{f}_{r}^{(1)} + \bm{f}_1$ and $\bm{f}_{r}^{(2)} + \bm{f}_2$, we have:
\begin{equation}
\label{eq:step1}
    \begin{bmatrix}
           \bm{f}_{r}^{(1)} + \bm{f}_1 \\
           \bm{f}_{r}^{(2)} + \bm{f}_2 
         \end{bmatrix} = \begin{bmatrix}
           \bm{f}_{r}^{(1)} \\
           \bm{f}_{r}^{(2)}
         \end{bmatrix} + \begin{bmatrix}
           \bm{f}_1 \\
           \bm{f}_2 
         \end{bmatrix}, 
             \begin{bmatrix}
           \bm{f}_{r}^{(1)} \\
           \bm{f}_{r}^{(2)}
         \end{bmatrix} \sim \mathcal{N}(\bm{0}, K_r) .
\end{equation}
Due to $\bm{f}_1 \indep \bm{f}_2$, where $\indep$ denotes $\bm{f}_1$ is independent of $\bm{f}_2$, we have:
\begin{equation}
\label{eq:step2}
    \begin{bmatrix}
           \bm{f}_1 \\
           \bm{f}_2 \\
         \end{bmatrix} \sim \mathcal{N} \left( \bm{0}, 
         \begin{bmatrix}
           K_1  & \bm{0}  \\
            \bm{0}    & K_2 
         \end{bmatrix} \right), 
\end{equation}
furthermore, because of the additive assumption in \Cref{eq:additive-assumption},
\begin{equation}
\label{eq:step3}
    \begin{bmatrix}
           \bm{f}_{r}^{(1)} \\
           \bm{f}_{r}^{(2)}
         \end{bmatrix} \indep 
         \begin{bmatrix}
           \bm{f}_1 \\
           \bm{f}_2 
         \end{bmatrix} .
\end{equation}
Combine \Cref{eq:step1,eq:step2,eq:step3}, we arrive at our key conclusion:
\begin{equation}
        \begin{bmatrix}
           \bm{f}_r^{(1)} {+} \bm{f}_1 \\
           \bm{f}_r^{(2)} {+} \bm{f}_2 \\
         \end{bmatrix} \sim \mathcal{N}\left(\bm{0}, 
         \left[
         \begingroup 
         \begin{array}{ll}
           K_r^{(11)} + K_1 & K_r^{(12)} \\
           K_r^{(21)}       & K_r^{(22)} + K_2
         \end{array}
         \endgroup
         \right]
         \right),
\label{eq:joint-distributation}         
\end{equation}
where $K_r$ is decomposed as follows:
\begin{equation*}
    K_r = \begin{bmatrix} 
    K_r^{(11)} & K_r^{(12)} \\ 
    K_r^{(21)} & K_r^{(22)} 
    \end{bmatrix}
\end{equation*}
in which $K_r^{(11)} \in \mathbb{R}^{n_1 \times n_1}, K_r^{(12)} \in \mathbb{R}^{n_1 \times n_2}, K_r^{(21)} \in \mathbb{R}^{n_2 \times n_1}, K_r^{(22)} \in \mathbb{R}^{n_2 \times n_2}$. The observation in \Cref{eq:joint-distributation} is crucial in two aspects: firstly, we can use a single covariance function and a global GP to model our objective, secondly and more importantly, this covariance function allows an efficient sharing mechanism between data points coming from different paths, although we cannot observe the decomposed function values at the shared vertex $r$, we can directly read out this sharing information from $K_r^{(12)}$.

\begin{algorithm}
\DontPrintSemicolon

\SetKwFunction{Kernel}{Kernel}
\SetKwFunction{Index}{Index}
\SetKwFunction{Dim}{Dim}
\SetKwFunction{Range}{Range}
\SetKwInOut{Input}{Input}
\SetKwInOut{Output}{Output}

\Input{The associated tree $\mathcal{T}=(V,E)$ of $f_\mathcal{T}$}
\Output{Add-Tree covariance function $k_{\mathcal{T}}$}


$k_{\mathcal{T}} \gets 0$ \;
\For{$v \gets V$}{
  $vi \gets \Index{v} $ \;
  $k_v^{d} \gets k_{v}^{\delta}(\bm{x}, \bm{x}') $ \tcc*{$k_{v}^{\delta}$ is 1 iff $\bm{x}$ and $\bm{x}'$ both have vertex $v$ in their paths}
  $si \gets vi + 1$ \tcc*{start index of v}
  $ei \gets vi + 1 + \Dim{v}$ \tcc*{end index of v}
  $k_v^{c} \gets k_c(\bm{x}_{si \leq i \leq ei}, \bm{x}'_{si \leq i \leq ei})$ \tcc*{$k_c$ is any p.s.d.\ covariance function}
  $k_v \gets k_v^{d} \times k_{v}^{c}$ \;
  $k_{\mathcal{T}} \gets k_{\mathcal{T}} + k_v$ 
}
\caption{Add-Tree Covariance Function}
\label{algo:kernel-construction}
\end{algorithm}

We summarize the  construction of an  Add-Tree covariance function in Algorithm~\ref{algo:kernel-construction}, where the value of \texttt{Index} comes from applying BFS to the associated tree structure $\mathcal{T}$ of $f_{\mathcal{T}}$ and \texttt{Dim} at $v$ is the dimension of the variable associated with vertex $v$. We provide implementation details in \cref{app:implementation-details}. In \Cref{app:additive-assumption-issues}, we discuss the case when additive assumption is not enough to model the objective function. 

\begin{proposition}
The Add-Tree covariance function defined by Definition~\ref{def:Add-Tree} is positive semi-definite for all tree-structured functions defined in Definition~\ref{def:TreeStructuredFunction} with the additive assumption satisfied. 
\end{proposition}
\begin{proof}
We will consider each term in \Cref{eq:add-tree-kernel-1} and demonstrate that it results in a positive semi-definite covariance function over the whole set of data points, not just the data points that follow the given path.  In particular, consider the p.s.d.\ covariance function $k_{v}^\delta(\bm{x}_{i'},\bm{x}_{j'}) = \begin{cases} 1 & \text{ if } v\in l_i \wedge v \in l_j \\
0 & \text{ otherwise}
\end{cases}$, for some vertex $v$.  We see that the product $k_v^c \times k_{v}^\delta$ defines a p.s.d.\ covariance function over the entire space of observations (since the product of two p.s.d.\ covariance functions is itself p.s.d.), and not just those sharing vertex $v$. In this way, we may interpret \Cref{eq:add-tree-kernel-1} as a summation over only the non-zero terms of a set of covariance functions defined over all vertices in the tree. As the resulting covariance function sums over p.s.d.\ covariance functions, and positive semi-definiteness is closed over positively weighted sums, the result is p.s.d.
\end{proof}

\section{BO FOR TREE-STRUCTURED FUNCTIONS}
\label{sec:bo}

In this section, we first describe how to perform the inference with our proposed Add-Tree covariance function, then we present a parallel algorithm for the optimization of the acquisition function.

\subsection{Inference with Add-Tree}

Given noisy observations $\mathcal{D}_n  = \{(\bm{x_i}, y_i)\}_{i=1}^{n}$, we would like to obtain the predictive distribution for the latent variable $f_{*\mathcal{T}}$ at a new input $\bm{x}_{*}$.
We begin with some notation. Let $p_{*}$ be the leaf selected by the categorical values of $\bm{x}_{*}$, $l_{*}$ be the path from the root $r$ to $p_{*}$. All $n$ inputs are collected in the design matrix $X \in \mathbb{R}^{n \times d}$, where the $i$-th row represents $\bm{x}_i \in \mathbb{R}^d$, and the targets and observation noise are aggregated in vectors $\bm{y} \in \mathbb{R}^n$ and $\bm{\sigma} \in \mathbb{R}^n$ respectively. 
Let $\Sigma=\diag(\bm{\sigma})$ be the noise matrix, where $\diag(\bm{\sigma})$ denotes a diagonal matrix containing the elements of vector $\bm{\sigma}$, $S=\{i \mid l_{*} \medcap l_{i} \neq \emptyset \}$, $I \in \mathbb{R}^n$ be the identity matrix, $M \in \mathbb{R}^{|S| \times n}$ be a selection matrix, which is constructed by removing the $j$-th row of $I$ if $j \notin S$, $X' = M X \in \mathbb{R}^{|S| \times d}$, $\bm{y'} = M \bm{y} \in \mathbb{R}^{|S|}$, $\Sigma^{'} = M \Sigma M^T \in \mathbb{R}^{|S| \times |S|}$. We can then write down the joint distribution of $f_{*\mathcal{T}}$ and $\bm{y'}$ as:
\begin{equation*}
    \begin{bmatrix}
           f_{*\mathcal{T}} \\
           \bm{y'}
    \end{bmatrix} \sim \mathcal{N}\left(\bm{0},
        \begin{bmatrix}
        k_{\mathcal{T}}(\bm{x}_{*}, \bm{x}_{*}) & k_{\mathcal{T}}(\bm{x}_{*}, X') \\
        k_{\mathcal{T}}( X', \bm{x}_{*}) & k_{\mathcal{T}}( X', X') + \Sigma'
        \end{bmatrix}
    \right) .
\end{equation*}
We note that this joint distribution has the same standard form~\citep{rasmussen2006} as in all GP-based BO, but that it is made more efficient by the selection of $X'$ based on the tree structure. 

The predictive distribution for $f_{*\mathcal{T}}$ is: 
\begin{equation}
    f_{*\mathcal{T}} \mid X',\bm{y'},\bm{x}_{*} \sim \mathcal{N}(\bar{f}_{*\mathcal{T}}, \Var(f_{*\mathcal{T}}))
\label{eq:predictive-distribution}
\end{equation}
where
\begin{eqnarray*}
\bar{f}_{*\mathcal{T}} & = &  k_{\mathcal{T}}(\bm{x}_{*}, X') \left[ k_{\mathcal{T}}( X', X') + \Sigma' \right]^{-1} \bm{y'}, \\
\Var(f_{*\mathcal{T}}) & = &  k_{\mathcal{T}}(\bm{x}_{*}, \bm{x}_{*}) \\ 
&& - k_{\mathcal{T}}(\bm{x}_{*}, X') [k_{\mathcal{T}}( X', X') + \Sigma']^{-1} k_{\mathcal{T}}( X', \bm{x}_{*}).
\end{eqnarray*}

Black-box calls to the objective function usually dominate the running time of BO, and the time complexity of fitting GP is of less importance. In \cref{app:time-complexity-inference}, we provide details on time complexity for our Add-Tree along with other related methods for completeness.

\subsection{Acquisition Function Optimization}                 

In BO, the acquisition function $u_{t-1}(\bm{x}| \mathcal{D})$, where $t$ is the current step of optimization, is used to guide the search for the optimum of our objective function. By trading off the exploitation and exploration, it is expected we can find the optimum using as few calls as possible to the objective. To get the next point at which we evaluate our objective function, we solve $\bm{x}_{t} = \argmax_{\bm{x} \in \mathcal{X}} u_{t-1}(\bm{x}| \mathcal{D})$. For noisy observations, GP-UCB~\citep{srinivas2010} has nice theoretical proprieties and explicit regret bounds for many commonly used covariance functions, and in this paper, we will use GP-UCB, which is defined as:
\begin{equation*}
    u_{t-1}(\bm{x}| \mathcal{D}) =  \mu_{t-1}(\bm{x}) + \beta_t^{1/2} \sigma_{t-1}(\bm{x}),
\end{equation*}
where $\beta_t$ are suitable constants, $\mu_{t-1}(\bm{x})$ and $\sigma_{t-1}(\bm{x})$ are the predictive posterior mean and variance at $\bm{x}$ from \Cref{eq:predictive-distribution}. 
Throughout the experiments in this paper, following~\citet{kandasamy2015}, we set $\beta_t = 0.2\tilde{d} \log (2t)$, in which $\tilde{d}$ denotes the dimension of the space where we optimize GP-UCB and is usually smaller than $d$ for a tree-structured function. We note the Add-Tree covariance function developed here can be combined with any other acquisition function. \cref{app:combine-with-other-acq} contains more details on combining other acquisition functions with Add-Tree.

A na\"{i}ve way to obtain the next evaluation point for a tree-structured function is to independently find $|P|$ optima, each one corresponding to the optimum of the associated function at a leaf, and then choose the best candidate across these optima. This approach is presented in~\citet{jenatton2017} and the authors there already pointed out this is too costly. Here we develop a much more efficient algorithm, which is dubbed as Add-Tree-GP-UCB, to find the next point and we summarise it in Algorithm~\ref{algo:bo}. By explicitly utilizing the associated tree structure $\mathcal{T}$ of $\bm{f}_{\mathcal{T}}$ and the additive assumption, the first two nested \texttt{for} loops can be performed in parallel.  Furthermore, as a by-product, each acquisition function optimization routine is now performed in a low dimensional space whose dimension is even smaller than the effective dimension. Time complexity analysis of \Cref{algo:bo} is given in \cref{app:time-complexity-algorithm2}.

\begin{algorithm}
\DontPrintSemicolon

\SetKwFunction{Kernel}{Kernel}
\SetKwFunction{Index}{Index}
\SetKwFunction{Dim}{Dim}
\SetKwFunction{Range}{Range}
\SetKwFunction{Pop}{Pop}
\SetKwFunction{Concat}{Concat}

\SetKwInOut{Input}{Input}
\SetKwInOut{Output}{Output}
\Input{The associated tree $\mathcal{T}=(V,E)$ of $f_\mathcal{T}$, Add-Tree covariance function $k_{\mathcal{T}}$ from Algorithm \ref{algo:kernel-construction}, paths $\{ l_i \}_{1\leq i \leq |P|}$ of $\mathcal{T}$}

$D_0 \gets \emptyset$\;
\For{$t \gets 1,\dots$}{
  \For{$v \gets V$}{
    $\bm{x}_{t}^{v} \gets \argmax_x  \mu_{t-1}^{v}(\bm{x}) + \sqrt{\beta_t} \sigma_{t-1}^{v}(\bm{x})$ \;
    $u_{t}^{v} \gets \mu_{t-1}^{v}(\bm{x}_{t}^{v}) + \sqrt{\beta_t} \sigma_{t-1}^{v}(\bm{x}_{t}^{v})$ \;
  }
  
  \For{$i \gets 1,\dots,|\mathcal{P}|$}{
    $U_t^{l_i} \gets \sum_{v \in l_i} u_{t}^{v} $ \tcc*{additive assumption from \Cref{eq:additive-assumption}}
  }
  $j \gets \argmax_{i} \{ U_t^{l_i} \mid i=1,\dots,|\mathcal{P}|\}  $ \;
  $\bm{x}_t \gets \medcup_{v \in l_j} \{\bm{x}_v\}$ \;
  $y_{t} \gets \bm{f}_{\mathcal{T}}(\bm{x}_t)$ \;
  $\mathcal{D}_t \gets \mathcal{D}_{t-1} \medcup \{(\bm{x}_t,y_t) \}$ \;
  Fitting GP using $\mathcal{D}_t$ to get $\{ (\mu_t^{v},\sigma_t^{v}) \}_{v \in V}$ using maximum likelihood \;
}

\caption{Add-Tree-GP-UCB}
\label{algo:bo}
\end{algorithm}

\section{EXPERIMENTS}
\label{sec:experiment}

In this section, we present results for two sets of experiments. 
To demonstrate the efficiency of our Add-Tree-GP-UCB, we first optimize the synthetic functions presented in~\citet{jenatton2017}, comparing to SMAC~\citep{hutter2011}, TPE~\citep{bergstra2011}, random search~\citep{bergstra2012}, standard GP-based BO from GPyOpt~\citep{gpyopt2016}, and the semi-parametric approach proposed in~\citet{jenatton2017}. To facilitate our following description, we refer to the above competing algorithms as \textbf{smac}, \textbf{tpe}, \textbf{random}, \textbf{gpyopt}, and \textbf{tree} respectively. We refer to our approach as \textbf{add-tree}. To verify our Add-Tree covariance function indeed enables sharing between different paths, we compare Add-Tree with independent GPs in the regression setting showing greater sample efficiency for our method. We then apply our method to the application of model compression for a three-layer fully connected neural network, outperforming competing methods.

For all GP-based BO, including \textbf{gpyopt}, \textbf{tree} and \textbf{add-tree}, we use the squared exponential (SE) covariance function: $k_{\text{SE}}(r) = \sigma \exp (- r^2 / 2 l^2)$. To optimize the parameters of Add-Tree, we maximize the marginal log-likelihood function of the corresponding GP. 
As for the numerical routine used in fitting the GPs and optimizing the acquisition functions, we use multi-started L-BFGS-B, as suggested by~\citet{kim2019}.
For all results in this section, we display the mean and twice the standard deviation of 10 independent runs.

We note the original code for~\citet{jenatton2017} is unavailable,\footnote{A request for the code was denied due to IP restrictions.} thus we have implemented their framework from scratch to obtain the results presented here. There are several hyper-parameters in their algorithm which are not specified in the publication.  To compare fairly with their method, we tune these hyper-parameters such that our implementation has a similar performance on the synthetic functions to that reported by \citet{jenatton2017}, and subsequently fix the hyper-parameter settings in the model compression task.

\subsection{Synthetic Experiments}
\label{sec:synthetic-function}

In our first experiment, we optimize the synthetic tree-structured function depicted in \Cref{fig:syn-function} and originally presented in~\citet{jenatton2017}. Non-shared variables including $x_4,x_5,x_6,x_7$ are defined in $[-1,1]$, all shared variables including $r_8,r_9$ are bounded in $[0,1]$, and all categorical variables including $x_1,x_2,x_3$ are binary. The dimension of $\mathcal{X}$ is $d=9$ and the effective dimension at any leaf is $2$. 

\begin{figure}[h]
\centering
\resizebox{0.95\columnwidth}{!}{
\begin{tikzpicture}[
    sibling distance        = 12em,
    level distance          = 5em,
    edge from parent/.style = {draw, -latex},
    every node/.style       = {font=\footnotesize},
    sloped,
    ]
  \node [root] {$x_1$}
    child {
      node [root] {$x_2,r_8$}
        child {
          node [env] {$x_4^2 + 0.1 + r_8$}
          edge from parent
            node[above] {0}
        }
        child {
          node [env] {$x_5^2 + 0.2 + r_8$}
          edge from parent
            node [above] {1}
        }
      edge from parent
        node [above] {0}
    }
    child {
      node [root] {$x_3,r_9$}
        child {
          node [env] {$x_6^2 + 0.3 + r_9$}
          edge from parent
            node[above] {0}
        }
        child {
          node [env] {$x_7^2 + 0.4 + r_9$}
          edge from parent
            node [above] {1}
        }
      edge from parent
        node [above] {1}
    };
\end{tikzpicture}
}
\caption{Synthetic Function From~\citeauthor{jenatton2017}}
\label{fig:syn-function}
\end{figure}
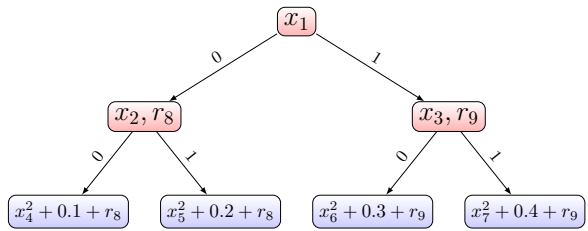

\Cref{fig:optimization-comparison} shows the optimization results for the different competing methods. The x-axis shows the iteration number and the y-axis shows the $log_{10}$ distance between the best minimum achieved so far and the known minimum value, which in this case is $0.1$. It is clear from \Cref{fig:optimization-comparison} that our method has a substantial improvement in performance compared with other algorithms. After 60 iterations, the $\log_{10}$ distance of~\textbf{tree} is still higher than -4, while our method can quickly obtain a much better performance in less than 20 iterations. Interestingly, our method performs substantially better than independent GPs, which will be shown later, while in~\citet{jenatton2017}, their algorithm is inferior to independent GPs. We note \textbf{gpyopt}\footnote{GPyOpt  \citep{gpyopt2016} is a state-of-the-art open source Bayesian optimization software package with support for categorical variables.} performs worst, and this is expected (recall \Cref{subsec:bo}). \textbf{gpyopt} encodes categorical variables using a one-hot representation, thus it actually works in a space whose dimension is $d'=d+c=12$, which is relatively high considering we have less than 100 data points. In this case, \textbf{gpyopt} behaves like random search, but in a 12-dimensional space instead of the 9-dimensional space of a na\"{i}ve random exploration.

\begin{figure}[h]
\centering
\includegraphics[width=0.95\columnwidth]{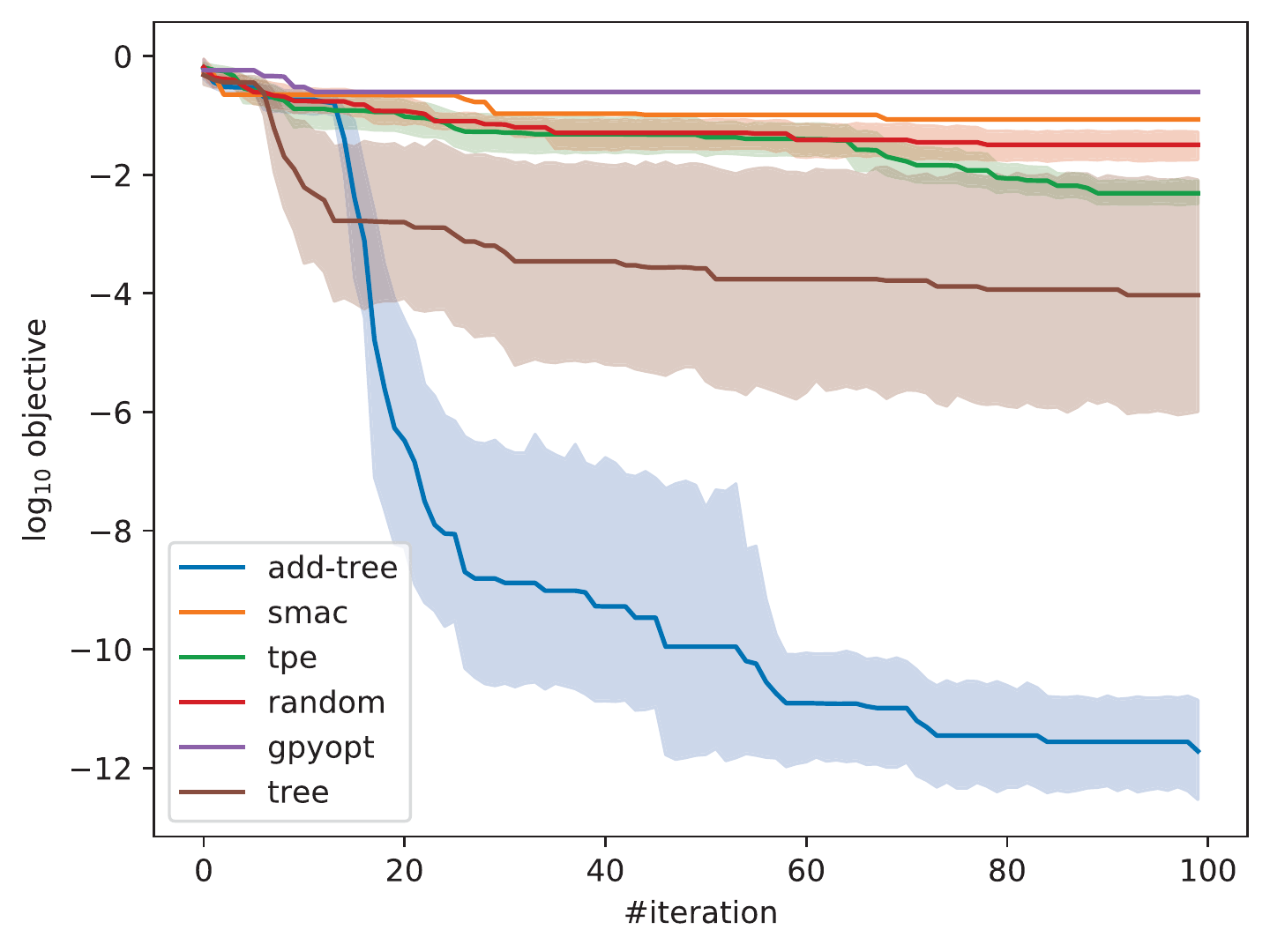}
\caption{Optimization Performance Comparison Of The Synthetic Function}
\label{fig:optimization-comparison}
\end{figure}

To show that Add-Tree allows efficient information sharing across different paths, we compare it with independent GPs and \textbf{tree} in a regression setting and results are shown in \Cref{fig:regression-comparison}. 
The training data is generated from the synthetic function in \Cref{fig:syn-function}: categorical values are generated from a Bernoulli distribution with $p=0.5$, continuous values are uniformly generated from their domains. The x-axis shows the number of generated training samples and the y-axis shows the $\log_{10}$ of Mean Squared Error (MSE) on a separate randomly generated data set with 50 test samples. 
From \Cref{fig:regression-comparison}, we see that Add-Tree models the response surface better even though independent GPs have more parameters. For example, to obtain a test performance of $10^{-4}$, Add-Tree needs only 24 observations, while independent GPs require 44 data points. If we just look at the case when we have 20 training samples, the absolute MSE of independent GPs is $10^{-1}$, while for Add-Tree, it is $10^{-3}$. The reason for such a huge difference is when training data are rare, some paths will have few data points, and Add-Tree can use the shared information from other paths to improve the regression model. This property of Add-Tree is valuable in BO settings. 

\begin{figure}[h]
\centering
\includegraphics[width=0.95\columnwidth]{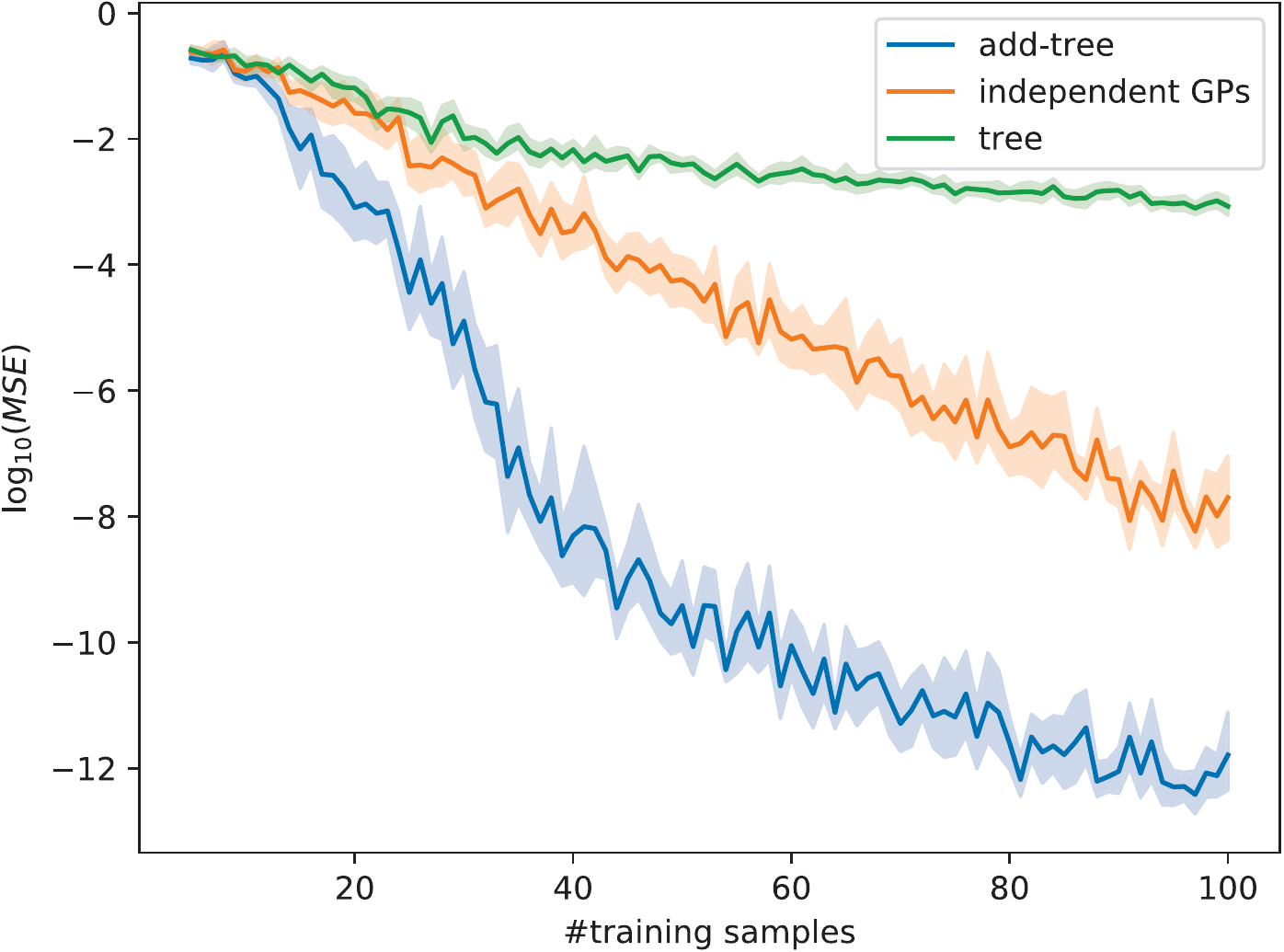}
\caption{Regression Performance Comparison Of The Synthetic Function}
\label{fig:regression-comparison}
\end{figure}

\subsection{Model Compression}
\label{sec:exp:model-compression}

Neural network compression is essential when it comes to deploying models on devices with limited memory and low computational resources. For parametric compression methods, like Singular Value Decomposition (SVD) and weight pruning (WP), it is necessary to tune their parameters in order to obtain the desired trade-off between model size and performance.  Existing publications on model compression usually determine parameters for a single compression method, and do not have an automated selection mechanism over different methods. By encoding this problem using a tree-structured function, different compression methods can now be applied to different layers and this formulation is more flexible than the current literature.

In this experiment, we apply our method to compress a three-layer fully connected network \textit{FC3} originally defined in \citet{Ma2019a}. FC3 has 784 input nodes, 10 output nodes, 1000 nodes for each hidden layer and is pre-trained on the MNIST dataset. For each layer, we find a compression method between SVD and WP, then optimize either the rank of the SVD or the pruning threshold of WP. We only compress the first two layers, because the last layer occupies 0.56\% of total weights. The rank parameters are constrained to be in $[10, 500]$ and the pruning threshold parameters are bounded in $[0,1]$. Following \citet{Ma2019a}, the objective function used in compressing FC3 is: 
\begin{equation}
\label{eq:obj-compression}
    \gamma \mathcal{L}(\tilde{f_\theta}, f^*) + R(\tilde{f_\theta}, f^*) ,
\end{equation}
where $f^*$ is the original FC3, $\tilde{f_\theta}$ is the compressed model using parameter $\theta$, $R(\tilde{f_\theta}, f^*)$ is the compression ratio and is defined to be the number of weights in the compressed network divided by the number of weights in the original network.  $\mathcal{L}(\tilde{f}_\theta, f^*) := \mathbb{E}_{x \sim P}(\|\tilde{f}_\theta(x) - f^*(x)\|_2^2)$, where $P$ is an empirical estimate of the data distribution. Intuitively, the $R$ term in \Cref{eq:obj-compression} prefers a smaller compressed network, the $\mathcal{L}$ term prefers a more accurate compressed network and $\gamma$ is used to trade off these two terms. In this experiment, $\gamma$ is fixed to be 0.01, and the number of samples used to estimate $\mathcal{L}$ is 50 following \citet{Ma2019a}. 
\Cref{fig:compress-fc3} shows the results of our method (Add-Tree) compared with other methods. 
For this experiment, although \textbf{smac}, \textbf{tpe} and \textbf{tree} all choose SVD for both layers at the end, our method converges significantly faster, once again demonstrating our method is more sample-efficient than other competing methods. 
We note \textbf{gpyopt} also has the worst performance among all other competing methods in this experiment.

\begin{figure}[h]
\centering
\includegraphics[width=0.95\columnwidth]{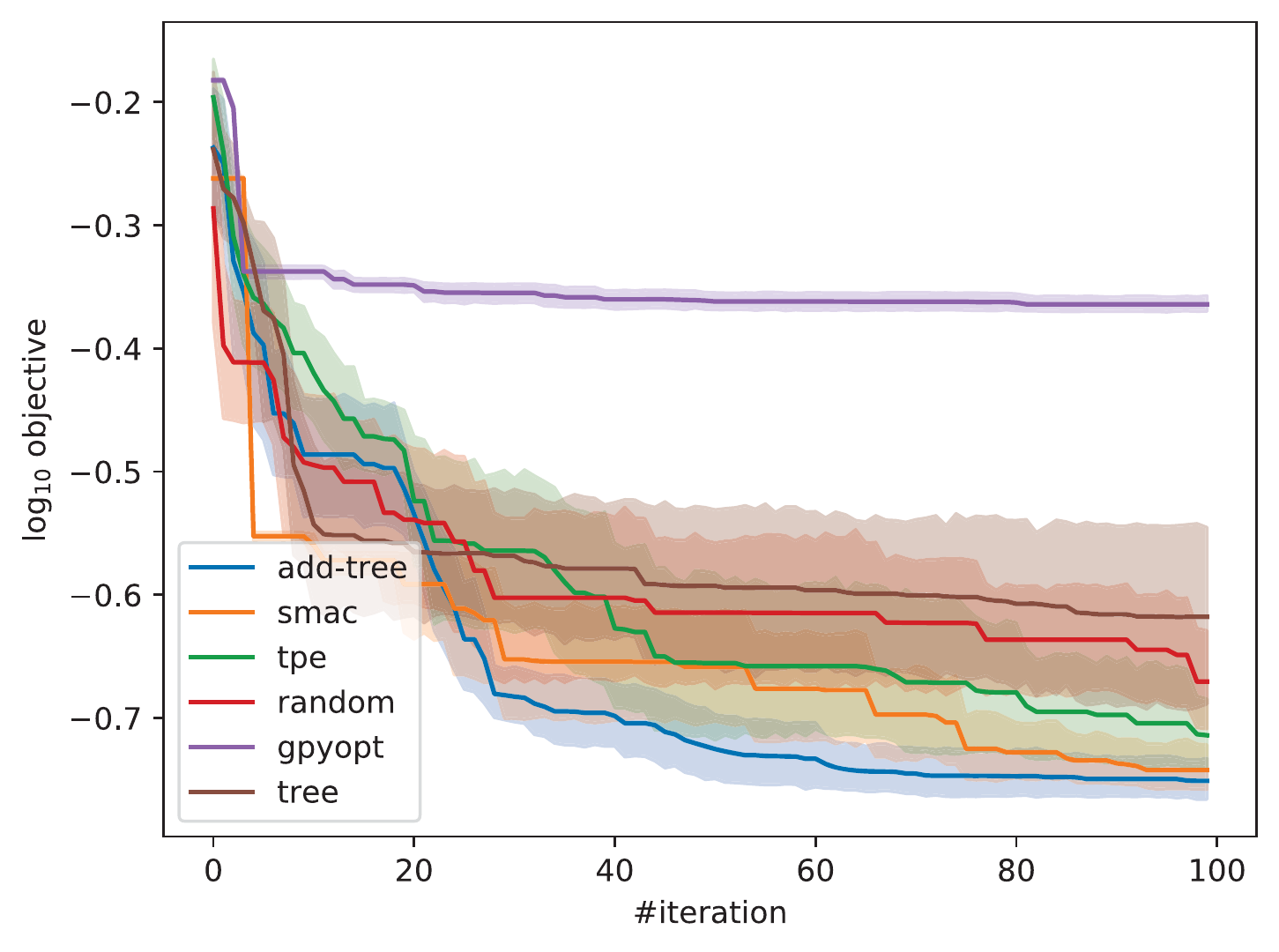}
\caption{Optimization Performance Comparison Of FC3 Compression}
\label{fig:compress-fc3} 
\end{figure}

Table~\ref{tab:wilcoxon-test} shows the results of pairwise Wilcoxon signed-rank tests\footnote{We used the Wilcoxon signed rank implementation from scipy.stats.wilcoxon with option (alternative == 'greater').} for the above two objective functions at different iterations.
In Table~\ref{tab:wilcoxon-test}, almost always performs significantly better than other competing methods (significance level $\alpha=0.05$), while no method is significantly better than ours.

\begin{table}[t]
\caption{\label{tab:wilcoxon-test}Wilcoxon Signed-Rank Test}
\centering
\resizebox{\columnwidth}{!}{
\tabcolsep=0.15cm
\begin{tabular}{ccccccc}
\toprule
\Gape[4pt]{\textbf{Experiment}} & \textbf{Iter} &
\textbf{smac} & \textbf{tpe} & \textbf{random} & \textbf{gpyopt} & \textbf{tree} \\
\midrule
 & 40 & 0.003 & 0.030 & 0.005 & 0.003 & 0.018\\
\cmidrule{2-7}
 & 60 & 0.003 & 0.003 & 0.003 & 0.003 & 0.003\\
\cmidrule{2-7}
\multirow{-3}{*}{\centering \makecell{synthetic \\ function}} & 80 & 0.003 & 0.003 & 0.003 & 0.003 & 0.003\\
\cmidrule{1-7}
 & 40 & 0.101 & 0.023 & 0.101 & 0.003 & 0.014\\
\cmidrule{2-7}
 & 60 & 0.037 & 0.018 & 0.011 & 0.003 & 0.008\\
\cmidrule{2-7}
\multirow{-3}{*}{\centering \makecell{model \\ compression}} & 80 & 0.166 & 0.005 & 0.003 & 0.003 & 0.006\\
\bottomrule
\end{tabular}
}
\end{table}

\section{CONCLUSION}
\label{sec:conclusion}

In this work, we have designed a covariance function that can explicitly utilize the problem structure, and demonstrated its efficiency on a range of problems. In the low data regime, our proposed Add-Tree covariance function enables a powerful information sharing mechanism, which in turn makes BO more sample-efficient compared with other model based optimization methods. 
Contrary to other GP-based BO methods, we do not impose restrictions on the structure of a conditional parameter space, greatly increasing the applicability of our method. 
We also directly model the dependencies between different observations under the framework of Gaussian Processes, instead of placing parametric relationships between different paths, making our method more flexible. In addition, we incorporate this structure information and develop a parallel algorithm to optimize the acquisition function. 
For both components of BO, our proposed method allows us to work in a lower dimensional space compared with the dimension of the original parameter space.

Empirical results on an optimization benchmark function and on a neural network compression problem show our method significantly outperforms other competing model based optimization algorithms in conditional parameter spaces, including SMAC, TPE and~\citet{jenatton2017}.

\subsection*{Acknowledgements}
Xingchen Ma is supported by Onfido.
This research received funding from the Flemish Government under the “Onderzoeksprogramma Artificiële Intelligentie (AI) Vlaanderen” programme.

\bibliography{mybib.bib}

\clearpage
\begin{appendices}
\crefalias{section}{appsec}

\section{Implementation Details}
\label{app:implementation-details}

The most important part in our implementation is a linear representation of the tree structure using breadth-first search (BFS). When BFS encounters a node, the linearization routine adds a tag, which is the rank of this node in its siblings, in front of the variables associated with this node, and this tag will be used to construct the delta kernel in \Cref{algo:kernel-construction}. \Cref{fig:bfs-example} shows the liner representation corresponding to the tree structure in \Cref{fig:example}. 

\begin{figure}[h]
\centering
\resizebox{0.95\columnwidth}{!}{
\begin{tikzpicture}

\fill[blue!40!white] (0,0) rectangle (1,1);
\draw (0,0) rectangle (1,1) node [pos=.5] {0};
\draw (1,0) rectangle (3,1) node [pos=.5] {$\bm{v}_r$};

\fill[blue!40!white] (3,0) rectangle (4,1);
\draw (3,0) rectangle (4,1) node [pos=.5] {0};
\draw (4,0) rectangle (6,1) node [pos=.5] {$\bm{v}_{p_1}$};

\fill[blue!40!white] (6,0) rectangle (7,1) ;
\draw (6,0) rectangle (7,1) node [pos=.5] {1};
\draw (7,0) rectangle (10,1) node [pos=.5] {$\bm{v}_{p_2}$};

\end{tikzpicture}
}
\caption{Linear representation of the tree structure corresponding to \Cref{fig:example}} 
\label{fig:bfs-example}
\end{figure}
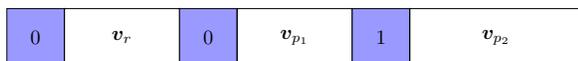

When linearizing an observation, we modify its tag using the following rules:
\begin{itemize}
    \item if a node is not associated with a continuous parameter, its tag is changed to a unique value
    \item if a node doesn't in this observation's corresponding path, its tag is changed to a unique value
    \item otherwise, we set the values after this tag to be the sub-parameter restricted to this node
\end{itemize}{}

For example, the linear representation of an observation $(0.1,0.2,0.3,0.4)$ falling into the left path is shown in \Cref{fig:bfs-example-left-observation} and the linear representation of an observation $(0.5,0.6,0.7,0.8,0.9)$ falling into the right path is shown in \Cref{fig:bfs-example-right-observation}.

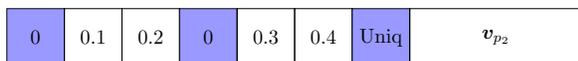
\begin{figure}[h]
\centering
\resizebox{0.95\columnwidth}{!}{
\begin{tikzpicture}

\fill[blue!40!white] (0,0) rectangle (1,1);
\draw (0,0) rectangle (1,1) node [pos=.5] {0};
\draw (1,0) rectangle (2,1) node [pos=.5] {0.1};
\draw (2,0) rectangle (3,1) node [pos=.5] {0.2};

\fill[blue!40!white] (3,0) rectangle (4,1);
\draw (3,0) rectangle (4,1) node [pos=.5] {0};
\draw (4,0) rectangle (5,1) node [pos=.5] {0.3};
\draw (5,0) rectangle (6,1) node [pos=.5] {0.4};

\fill[blue!40!white] (6,0) rectangle (7,1) ;
\draw (6,0) rectangle (7,1) node [pos=.5] {Uniq};
\draw (7,0) rectangle (10,1) node [pos=.5] {$\bm{v}_{p_2}$};

\end{tikzpicture}
}
\caption{Linear representation of an observation falling into the lower path corresponding to \Cref{fig:example}} 
\label{fig:bfs-example-left-observation}
\end{figure}

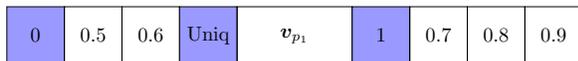
\begin{figure}[h]
\centering
\resizebox{0.95\columnwidth}{!}{
\begin{tikzpicture}

\fill[blue!40!white] (0,0) rectangle (1,1);
\draw (0,0) rectangle (1,1) node [pos=.5] {0};
\draw (1,0) rectangle (2,1) node [pos=.5] {0.5};
\draw (2,0) rectangle (3,1) node [pos=.5] {0.6};

\fill[blue!40!white] (3,0) rectangle (4,1);
\draw (3,0) rectangle (4,1) node [pos=.5] {Uniq};
\draw (4,0) rectangle (6,1) node [pos=.5] {$\bm{v}_{p_1}$};

\fill[blue!40!white] (6,0) rectangle (7,1) ;
\draw (6,0) rectangle (7,1) node [pos=.5] {1};
\draw (7,0) rectangle (8,1) node [pos=.5] {0.7};
\draw (8,0) rectangle (9,1) node [pos=.5] {0.8};
\draw (9,0) rectangle (10,1) node [pos=.5] {0.9};

\end{tikzpicture}
}

\caption{Linear representation of an observation falling into the upper path corresponding to \Cref{fig:example}} 
\label{fig:bfs-example-right-observation}
\end{figure}

Based on this linear representation, it is now straightforward to compute the covariance function, which is constructed in \Cref{algo:kernel-construction},  between any two observations. We note the dimension of such a linear representation has order $\mathcal{O}(d)$, where $d$ is the dimension of the original parameter space, thus there is little overhead compared with other covariance functions in existing GP libraries, such as RBF or Matern in GPyOpt. 



\section{Time Complexity Analysis of Inference with Add-Tree}
\label{app:time-complexity-inference}


We analyse the time complexity of inference using our proposed Add-Tree covariance function by recursion. Without loss of generality, let the tree structure $\mathcal{T}$ be a binary tree. If the root node $r$ of $\mathcal{T}$ is associated with a continuous parameter, which means this parameter is shared by all paths. In this worst case, the gram matrix in \Cref{eq:joint-distributation} is dense and structureless, and the complexity will be $\mathcal{O}(n^3)$, where $n$ is the number of observations. Otherwise, when $r$ is not associated to any continuous parameter, let $n_l$ and $n_r$ be the number of samples falling into the left path and the right path respectively, we have $T(r) = T(rl) + T(rr)$, where $rl$ is the left child of $r$, $rr$ is the right child of $r$, $T(r),T(rl),T(rr)$ are the running time at nodes $r,rl,rr$ respectively. Because the worst-case running time at nodes $rl$ and $rr$ is $\mathcal{O}(n_l^3)$ and $\mathcal{O}(n_r^3)$ respectively, we have $T(r)=\mathcal{O}(n_l^3 + n_r^3)$. 

Table~\ref{tbl:time-complexity} summarizes the worst-case inference time complexity comparison of our Add-Tree covariance function and other related methods. In Table~\ref{tbl:time-complexity}, $n_i$ is number of observations falling into path $l_i$ and $n = \sum_{1 \leq i \leq |P|} n_i^{3}$.
In general, Add-Tree performs the worst among these three methods from the aspect of time complexity. However, due to the explicit sharing mechanism, our approach requires fewer black-box calls to the expensive objective function, which typically dominates the computational cost of the GP model.

\begin{table*}[h]
\caption{Inference Time Complexity Comparison}
\label{tbl:time-complexity}
\centering
\begin{tabular}{lll}
\toprule
\Gape[4pt]{\textbf{Method}} & \textbf{Share?} & \textbf{Complexity} \\
\midrule
Independent & no & $\mathcal{O}(\sum_{1 \leq i \leq |P|} n_i^{3})$ \\ 
\citeauthor{jenatton2017} & yes & $\mathcal{O}(\sum_{1 \leq i \leq |P|} n_i^{3} + |V|^3)$ \\
Add-Tree    & yes & $\begin{cases} \mathcal{O}(n_l^3 + n_r^3) & \text{ if no sharing continuous parameter at } r  \\
\mathcal{O}(n^3) & \text{ otherwise}
\end{cases}$\\
\bottomrule
\end{tabular}

\end{table*}

\section{Time Complexity Analysis of \Cref{algo:bo}}
\label{app:time-complexity-algorithm2}
W.l.o.g, let the tree structure $\mathcal{T}$ be a perfect binary tree, the depth of this tree be $h$, and suppose all nodes are associated with a $d_u$ dimensional vector. Then $|P| = 2^{h-1}$ and $|V| = 2^h - 1$. The running time of searching in every leaves in a na\"{i}ve way is $|P| + |P| \mathcal{O}(h^2 d_u^2) = 2^{h-1} + 2^{h-1} \mathcal{O}(h^2 d_u^2)$. 
The running time of \Cref{algo:bo} is $|V| \mathcal{O}(d_u^2) + |P| h + |P| =  2^{h-1} + 2^{h-1} (h + 2 \mathcal{O}(d_u^2)) - \mathcal{O}(d_u^2)$, here we keep the constants just for clarity. It is clear that \Cref{algo:bo} has a substantial advantage over a na\"{i}ve method when $h \ge 2$. We note the complexity of BFGS is $\mathcal{O}(n^2)$, where $n$ is the dimension of the parameter.

\section{Combine With Other Acquisition Functions}
\label{app:combine-with-other-acq}

Add-Tree covariance function itself can be combined with any other acquisition functions and enables efficient information sharing. To efficiently optimize the acquisition function using \Cref{algo:bo}, it is required the acquisition function has additive structure, otherwise the two-step approach in ~\citet{jenatton2017} can be used.

\section{When Additive Assumption is Not Enough}
\label{app:additive-assumption-issues}

For objective functions with known additive structure, our proposed Add-Tree covariance function usually performs the best. If there is an interaction effect between the variables along a single path in the tree structure, we can combine the method proposed in \citet{duvenaud2011} by including higher order additive kernels for these variables. 
To illustrate the covariance function design in this case, we again take the tree-structured function in \Cref{fig:example} as an example. Since there is an interaction effect between $\bm{v}_r$ and $\bm{v}_{p1}$, the latent variables associated to $f_{p_1, \mathcal{T}}$ is decomposed as $\bm{f}_{r}^{(1)} + \bm{f}_1 + \bm{f}_{r1}$, where $\bm{f}_{r1}$ is the interaction term between $\bm{v}_r$ and $\bm{v}_{p1}$. Similarity, the latent variables associated to $f_{p_2, \mathcal{T}}$ is $\bm{f}_{r}^{(2)} + \bm{f}_2 + \bm{f}_{r2}$. Similar to \Cref{eq:step1}, we have:

\begin{equation}
\label{eq:app-interaction-step1}
    \begin{bmatrix}
           \bm{f}_{r}^{(1)} + \bm{f}_1 + \bm{f}_{r1}\\
           \bm{f}_{r}^{(2)} + \bm{f}_2 + \bm{f}_{r2}
         \end{bmatrix} = \begin{bmatrix}
           \bm{f}_{r}^{(1)} \\
           \bm{f}_{r}^{(2)}
         \end{bmatrix} + \begin{bmatrix}
           \bm{f}_1 \\
           \bm{f}_2 
         \end{bmatrix} + \begin{bmatrix}
           \bm{f}_{r1} \\
           \bm{f}_{r2} 
         \end{bmatrix}.
\end{equation}

To model $\bm{f}_{r1}$ and $\bm{f}_{r2}$ separately, we can use product covariance functions $k_{r} k_1$ and $k_{r} k_2$ respectively.
Without further assumptions, it is not clear how to model the covariance between $\bm{f}_{r1}$ and $\bm{f}_{r2}$. A visualization is shown in \Cref{eq:app-gram-matrix-interaction}. In this case, a safe choice is to set these covariance to be zero, because over-estimating the covariance will confuse the GP, and the price paid for ignoring these covariance here is we lose some potential sample-efficiency. 

\def\InteractioGramMatrix{\tikz[baseline=.1ex]{

\draw (-1,-1) rectangle (0,0) node [pos=.5] {?};
\draw (0,0) rectangle (1,1) node [pos=.5] {?};
\fill[blue!40!white] (0,-1)  rectangle (1,0);
\draw (0,-1) rectangle (1,0) node [pos=.5] {$K_{r2}$};
\fill[blue!40!white] (-1,0)  rectangle (0,1);
\draw (-1,0) rectangle (0,1) node [pos=.5] {$K_{r1}$};
}
}

\begin{equation}
\label{eq:app-gram-matrix-interaction}
    \begin{bmatrix}
           \bm{f}_{r1} \\
           \bm{f}_{r2} 
         \end{bmatrix} \sim \mathcal{N}\left(\bm{0}, \InteractioGramMatrix \right).
\end{equation}{}

Combine \Cref{eq:joint-distributation,eq:app-interaction-step1,eq:app-gram-matrix-interaction}, we obtain the joint distribution of a tree-structured function with interaction effects:
\begin{equation}
        \begin{bmatrix}
           \bm{f}_{r}^{(1)} + \bm{f}_1 + \bm{f}_{r1}\\
           \bm{f}_{r}^{(2)} + \bm{f}_2 + \bm{f}_{r2}
         \end{bmatrix} \sim \mathcal{N}\left(\bm{0}, 
         \left[
         \begingroup 
         \begin{array}{ll}
           K_{11} & K_r^{(12)}  \\
           K_r^{(21)} & K_{22}
         \end{array}
         \endgroup
         \right]
         \right),
\label{eq:app-joint-distributation}         
\end{equation}
where $K_{11} =K_r^{(11)} + K_1 + K_{r1}$ and $K_{22}=K_r^{(22)} + K_2 + K_{r2}$. To implement the Add-Tree covariance function with interaction effects, the linear representation presented in \Cref{app:implementation-details} remains unchanged. For \Cref{algo:kernel-construction}, we only need to construct an extra term by multiplying the corresponding delta covariance function with the interaction terms we are interested in, and append this extra term in the final covariance function.

\end{appendices}

\end{document}